\documentclass[11pt]{amsart}

\usepackage{algorithm}
\usepackage{placeins}
\usepackage{algpseudocode}
\usepackage{epigraph}
\usepackage{mathtools}
\usepackage{natbib}
\usepackage{latexsym}
\usepackage{mathrsfs}
\usepackage{pstricks}
\usepackage{listings}
\usepackage{hyperref}
\usepackage{pgfplots}
\usepackage{appendix}
\usepackage{pst-all}
\usepackage{pstricks,pst-plot}
\pgfplotsset{compat=1.18} 
\usepackage{mathdots}
\usepackage{xcolor}
\usepackage{enumitem}
\usepackage{todonotes}
\usepackage{amssymb,bm}
\usepackage{amsmath}
\usepackage{amsfonts}
\usepackage{amsthm}
\usepackage{tikz}
\usepackage{epsfig}
\usepackage{verbatim}
\usepackage{float}
\usepackage{tabularx}
\usepackage{bm}
\usepackage{mathtools}
\usepackage{blkarray, bigstrut}
\usepackage{subcaption}

\definecolor{codegray}{rgb}{0.5,0.5,0.5}
\lstdefinestyle{mystyle}{
    commentstyle=\color{codegreen},
    keywordstyle=\color{magenta},
    numberstyle=\tiny\color{codegray},
    stringstyle=\color{codepurple},
    basicstyle=\ttfamily\footnotesize,
    breakatwhitespace=false,         
    breaklines=true,                 
    captionpos=b,                    
    keepspaces=true,                 
    numbers=left,                    
    numbersep=5pt,                  
    showspaces=false,                
    showstringspaces=false,
    showtabs=false,                  
    tabsize=2
}
  \lstdefinelanguage{GAP}{
    basicstyle=\ttfamily,
    keywords={true, false, function, return, fail, if, in, while, do, od, else, elif, fi, break, continue},
    keywordstyle=\color{blue}\bfseries,
    otherkeywords={% Operators
      >, <, ==
    },
    breaklines=true,      
    identifierstyle=\color{black},
    sensitive=True,
    comment=[l]{\#},
    commentstyle=\color{cyan},
    stringstyle=\color{red},
    morestring=[b]',
    morestring=[b]"
  }

\providecommand{\U}[1]{\protect\rule{.1in}{.1in}}

\newcolumntype{Y}{>{\raggedleft\arraybackslash}X}
\def\bc{{\mathbb{C}}}

\def\br{{\mathbb{R}}}

\def\br{\mathbb R}

\def\vs{\vskip.3cm}

\def\t2deg{\mathbb T^2\text{\rm -deg}}

\def\s1deg{S^1\text{\rm -deg}}

\newrgbcolor{violet}{.6 .1 .8}
  \newrgbcolor{lightyellow}{1 1 .8}
  \newrgbcolor{lightblue}{.80 1 1}
  \newrgbcolor{mygreen}{0 .66 .05}
  \definecolor{mygreen}{rgb}{0,.66,.05}
  \definecolor{lightyellow}{rgb}{1,1,.80}
  \newrgbcolor{orange}{1 .6 0}
  \newrgbcolor{GreenYellow}{.85 1 .31}
  \newrgbcolor{Yellow}{1  1  0}
  \newrgbcolor{Goldenrod}{1  .90  .16}
  \newrgbcolor{Dandelion}{1  .71  .16}
  \newrgbcolor{Apricot}{1  .68  .48}
  \newrgbcolor{Peach}{1  .50  .30}
  \newrgbcolor{Melon}{1  .54  .50}
  \newrgbcolor{YellowOrange}{1  .58  0}
  \newrgbcolor{Orange}{1  .39  .13}
  \newrgbcolor{BurntOrange}{1  .49  0}
  \newrgbcolor{Bittersweet}{1.  .4300  .24}
  \newrgbcolor{RedOrange}{1  .23  .13}
  \newrgbcolor{Mahogany}{1.  .4475  .4345}
  \newrgbcolor{Maroon}{1.  .4084  .5376}
  \newrgbcolor{BrickRed}{1.  .3592  .3232}
  \newrgbcolor{Red}{1  0  0}
  \newrgbcolor{OrangeRed}{1  0  .50}
  \newrgbcolor{RubineRed}{1  0  .87}
  \newrgbcolor{WildStrawberry}{1  .04  .61}
  \newrgbcolor{CarnationPink}{1  .37  1}
  \newrgbcolor{Salmon}{1  .47  .62}
  \newrgbcolor{Magenta}{1  0  1}
  \newrgbcolor{VioletRed}{1  .19  1}
  \newrgbcolor{Rhodamine}{1  .18  1}
  \newrgbcolor{Mulberry}{.6668  .1180  1.}
  \newrgbcolor{RedViolet}{.9538  .4060  1.}
  \newrgbcolor{Fuchsia}{.5676  .1628  1.}
  \newrgbcolor{Lavender}{1  .52  1}
  \newrgbcolor{Thistle}{.88  .41  1}
  \newrgbcolor{Orchid}{.68  .36  1}
  \newrgbcolor{DarkOrchid}{.60  .20  .80}
  \newrgbcolor{Purple}{.55  .14  1}
  \newrgbcolor{Plum}{.50  0  1}
  \newrgbcolor{Violet}{.98 .15 .95}
  \newrgbcolor{RoyalPurple}{.25  .10  1}
  \newrgbcolor{BlueViolet}{.84  .38  .98}
  \newrgbcolor{Periwinkle}{.43  .45  1}
  \newrgbcolor{CadetBlue}{.38  .43  .77}
  \newrgbcolor{CornflowerBlue}{.35  .87  1}
  \newrgbcolor{MidnightBlue}{.4414  .9259  1.}
  \newrgbcolor{NavyBlue}{.06  .46  1}
  \newrgbcolor{RoyalBlue}{0  .50  1}
  \newrgbcolor{Blue}{0  0  1}
  \newrgbcolor{Cerulean}{.06  .89  1}
  \newrgbcolor{Cyan}{0  1  1}
  \newrgbcolor{ProcessBlue}{.04  1  1}
  \newrgbcolor{SkyBlue}{.38  1  .88}
  \newrgbcolor{Turquoise}{.15  1  .80}
  \newrgbcolor{TealBlue}{.1572  1.  .6668}
  \newrgbcolor{Aquamarine}{.18  1  .70}
  \newrgbcolor{BlueGreen}{.15  1  .67}
  \newrgbcolor{Emerald}{0  1  .50}
  \newrgbcolor{JungleGreen}{.01  1  .48}
  \newrgbcolor{SeaGreen}{.31  1  .50}
  \newrgbcolor{Green}{0  1  0}
  \newrgbcolor{ForestGreen}{.1992  1.  .2256}
  \newrgbcolor{PineGreen}{.3100  1.  .5575}
  \newrgbcolor{LimeGreen}{.50  1  0}
  \newrgbcolor{YellowGreen}{.56  1  .26}
  \newrgbcolor{SpringGreen}{.74  1  .24}
  \newrgbcolor{OliveGreen}{.6160  1.  .4300}
  \newrgbcolor{RawSienna}{.53  .28  .16}
  \newrgbcolor{Sepia}{1.  .7510  .70}
  \newrgbcolor{Brown}{.41  .25  .18}
  \newrgbcolor{TAN}{.86  .58  .44}
  \newrgbcolor{Gray}{1.  1.  1.}
  \newrgbcolor{Black}{1  1  1}
  \newrgbcolor{White}{1  1  1}

\newtheorem{proposition}{Proposition}[section]

\newtheorem{remark}{Remark}[section]

\newtheorem{remark-definition}{Remark and Definition}[section]
\newtheorem{rem-not}{Remark and Notation}[section]

\begin{document}

\title[Symmetry Inference in Chaotic Attractors]{A Bayesian Framework for Symmetry Inference in Chaotic Attractors} 

\author{ Ziad Ghanem}\address{Department of Mathematical Sciences, University of Texas at Dallas, Richardson, TX 75080, USA}
\email{Ziad.Ghanem@utdallas.edu}

\author{Chang Hyunwoong}
\address{Department of Mathematical Sciences, University of Texas at Dallas, Richardson, TX 75080, USA}
\email{hwchang@utdallas.edu}

\author{Preskella Mrad}\address{Department of Mathematical Sciences, University of Texas at Dallas, Richardson, TX 75080, USA}
\email{Preskella.Mrad@utdallas.edu}

\date{}

\begin{abstract}
Detecting symmetry from data is a fundamental problem in signal analysis, providing insight into underlying structure and constraints. When data emerge as trajectories of dynamical systems, symmetries encode structural properties of the dynamics that enable model reduction, principled comparison across conditions, and detection of regime changes. While recent optimal-transport methods provide practical tools for data-driven symmetry detection in this setting, they rely on deterministic thresholds and lack uncertainty quantification, limiting robustness to noise and ability to resolve hierarchical symmetry structures. We present a Bayesian framework that formulates symmetry detection as probabilistic model selection over a lattice of candidate subgroups, using a Gibbs posterior constructed from Wasserstein distances between observed data and group-transformed copies. We establish three theoretical guarantees: $(i)$ a Bayesian Occam’s razor favoring minimal symmetry consistent with data, $(ii)$ conjugation equivariance ensuring frame-independence, and $(iii)$ stability bounds under perturbations for robustness to noise. Posterior inference is performed via Metropolis–Hastings sampling and numerical experiments on equivariant dynamical systems and synthetic point clouds demonstrate accurate symmetry recovery under high noise and small sample sizes. An application to human gait dynamics reveals symmetry changes induced by mechanical constraints, demonstrating the framework's utility for statistical inference in biomechanical and dynamical systems.
\end{abstract}


\maketitle

\section{Introduction} \label{sec:introduction}
The reliable extraction of coherent signals from noisy observational data is a persistent challenge that motivates the development of novel mathematical and statistical methods. To formalize the notion of coherence in this context, we appeal to the language of algebra: a data set $X$ is said to admit the \textbf{symmetries} of a group $G$ if its structure is invariant under the action of $G$. Successful identification of symmetry from observations enables $(i)$ model reduction and sample‑efficient learning via fundamental domains, $(ii)$ principled comparison across experiments or operating conditions by quotienting out redundant orbits, and $(iii)$ robust detection of regime changes when symmetry is gained, broken, or only partially expressed in noisy measurements. 
\vs
The systematic exploitation of symmetry has a long history in the study of differential equations, culminating in modern equivariant bifurcation theory. This theory employs a range of tools, from the algebraic methods of singularity theory for the local classification of symmetry-breaking phenomena \cite{GolubitskySchaeffer84, GolubitskyStewartSchaeffer, GolubitskyStewart03} to the topological methods of equivariant degree theory for proving the global existence of solutions \cite{survey, book-new, AED, GarciaGhanemKrawcewicz, Ghanem, GhanemCraneLiu}. These advanced techniques all build upon the foundational nineteenth-century work of Sophus Lie, who first used continuous groups to systematically study differential equations \cite{Olver}. For example, the long-term behavior of a dynamical system is represented by sets towards which all trajectories converge, called \textbf{attractors}. The symmetric invariances of an attractor encode structural constraints on the underlying dynamics by organizing bifurcations, linking states into conjugate families, and constraining the forms of possible solutions. Since, in practice, the full attractor is inaccessible, its properties must be inferred from observed trajectories that serve as finite, often noisy proxies. 
Foundational work by Chossat and Golubitsky~\cite{ChossatGolubitsky88} demonstrated how sequences of symmetry-breaking bifurcations generate multiple conjugate chaotic attractors that subsequently merge into larger, more symmetric sets. Melbourne et al. extended this framework by establishing connectedness properties and group-theoretic constraints for symmetric attractors without requiring restrictive dynamical assumptions \cite{Melbourne1993}. These developments sharpen a central question for data analysis: \emph{given only partial and noisy observations, can we reliably infer the symmetry structure of the underlying and unobserved attractor?}
\vs 
The challenge of numerically identifying symmetries from partial attractor data motivated the ``symmetry detectives'' framework introduced by Barany et al.~\cite{Barany93}. The approach involves the construction of group-equivariant observables tailored to a specified parent group and its irreducible representations. Integrating these observables over a thickened attractor yields a diagnostic point whose isotropy reveals the symmetry. Ashwin and Nicol~\cite{AshwinNicol97} reformulated this strategy for functions on an intermediate ``observation space,'' enabling inference from partial measurements such as Poincaré sections.
In a companion study, Ashwin and Tomes~\cite{AshwinTomes97} demonstrated the method experimentally on a physical system of four coupled oscillators with $S_4$ symmetry. Despite its originality, practical deployment of the symmetry-detectives approach requires
\emph{prior knowledge of both the parent group and an equivariant generating map} and depends on expert-designed, group-specific observables.
\vs 
More recently, data‑driven methods have leveraged optimal transport to test for invariances directly from sampled trajectories, overcoming the need for handcrafted observables. Cisternas~\cite{Cisternas2025} represents sampled trajectories as empirical measures and computes Wasserstein distances between the data and its group-transformed copies. A subgroup is accepted if all element-wise distances fall below a user-specified threshold. This framework is elegant and broadly applicable \emph{once a candidate group action is specified} but its reliance on deterministic thresholding is sensitive to noise and nested subgroup hierarchies, and the method provides no uncertainty quantification. Moreover, practical application \emph{still presumes a priori knowledge of a parent group} in order to define the lattice of candidate isotropy groups for testing.
\vs
These limitations motivate a framework that is uncertainty‑aware and robust to observational noise. Unlike approaches that require commitment to a parent symmetry group or an explicit system of differential equations, our framework identifies symmetries solely from the observed data structure. We formulate symmetry identification as Bayesian model comparison over the lattice of subgroup hypotheses, replacing arbitrary thresholds with calibrated probabilistic inference. Our approach centers on a group-orbit Gibbs posterior defined in terms of the average squared Wasserstein distance between observed data and its group-transformed copies. In Proposition \ref{prop:occams_razor}, we prove that this construction naturally implements a Bayesian Occam's razor, automatically favoring the simplest group consistent with the data. We establish two additional theoretical guarantees. First, the cost function is conjugation-invariant, ensuring consistency across physically equivalent datasets (see Proposition \ref{prop:conjugation_invariance}). Second, we derive stability bounds under data perturbations, providing robustness to measurement noise and enabling principled uncertainty quantification (see Proposition \ref{prop:stability}).
Posterior inference is performed via Metropolis–Hastings sampling \cite{Hastings70, Metropolis53}, a cornerstone of modern computational statistics, allowing our algorithm to efficiently explore the discrete subgroup space and return a full posterior distribution. Because the framework does not rely on the existence of an underlying dynamical system, it applies equally well to rotational symmetries in observed trajectories from dynamical systems and synthetic point clouds; Section~\ref{sec:validation} illustrates these capabilities on canonical dihedral examples as well as real‑world biomechanical time series.
\vs
\noindent \textbf{Organization:} The remainder of this paper is organized as follows. Section \ref{sec:methodology} introduces the theoretical foundations and formalizes the symmetry inference problem. Section \ref{sec:bayesian_framework} presents our Bayesian framework, including the Gibbs posterior construction and MCMC sampling strategy. We also provide theoretical guarantees such as conjugation equivariance and stability under perturbations. Section \ref{sec:validation} reports numerical experiments on synthetic and real-world datasets, demonstrating the robustness and interpretability of our approach. For the convenience of the reader, we include in Appendix \ref{sec:transformation_groups} a primer on transformation groups with a particular focus in Appendix \ref{app:dihedral} on the action of the dihedral groups on the complex plane.

\section{Methodology}\label{sec:methodology}

\subsection{Preliminaries: The Structure of Symmetric Attractors}
Consider a discrete dynamical system generated by a continuous map $f: \br^d \to \br^d$. The system is said to be equivariant under the action of a compact Lie group $G$ if the map commutes with all group elements, i.e.
\begin{align*}
    f(g x) = g f(x) \text{ for all } g \in G \text{ and } x \in \br^d.
\end{align*}
The dynamics of the system can be inferred by choosing an initial point $x_0 \in \br^d$ and observing the trajectory formed by the simple recursion rule $x_{n+1} := f(x_n)$. An \emph{attractor} of the system is a set $A \subset \mathbb{R}^d$ toward which trajectories converge for a wide range of initial conditions. The \emph{isotropy} of an attractor $A$ is the largest subgroup $\Sigma(A) \le G$ that leaves $A$ invariant, i.e.
\[
\Sigma(A) := \{ g \in G: gA = A\} \; \text{ where } \; gA := \{ga : a \in A\}.
\]
In practice, the attractor is never observed directly. Instead, we have access only to a finite trajectory $x_{n+1} = f(x_n) +   \epsilon_n$, where the term $\epsilon_n \in \br$ is used to account for possible observational noise, represented as a point cloud $X = \{ x_i \}_{i=1}^N.$
With no knowledge of the generating map $f$ or its equivariance, our goal is to infer the symmetries of the underlying attractor $A$ from this empirical proxy $X$.
 
\subsection{The Wasserstein Distance as a Symmetry Metric}

To quantify invariance, we associate to the point cloud $X$ the empirical measure
\[
\mu_X = \frac{1}{N} \sum_{i=1}^N \delta_{x_i},
\]
where $\delta_{x_i}$ denotes the Dirac measure at $x_i$, such that for any measurable set $Y \subset \mathbb{R}^d$, one has $\mu_X(Y) = \frac{1}{N} | X \cap Y |$.
\vs
Following \cite{Cisternas2025}, we measure the discrepancy between $\mu_X$ and its transformed copy $\mu_{\sigma X}$, where $\sigma X := \{ \sigma x_i \}_{i=1}^N,
$ using the $p$-Wasserstein distance. For two empirical measures with equal cardinality and uniform weights, this distance is given by the solution of the following optimal transport problem
\begin{align}
\label{eq:wasserstein}
    W_p(X, \sigma X) = \Bigg( \min_{P \in \Pi_N} \sum_{i=1}^N \sum_{j=1}^N P_{ij} \| x_i - \sigma x_j \|^p \Bigg)^{1/p},
\end{align}
where $\Pi_N$ is the set of all $N \times N$ transport plan matrices $P: \br^N \rightarrow \br^N$ satisfying the element-wise constraint $\sum_{k=1}^N P_{ik} = \sum_{k=1}^N P_{kj} = \frac{1}{N}$ for all $1 \le i,j \le N$.
If the attractor admits isotropy $\Sigma$, then for every $\sigma \in \Sigma$, the distance $W_p(X,\sigma X)$ will be close to zero, up to sampling error and observational noise. 

\section{A Bayesian Framework for Symmetry Inference}\label{sec:bayesian_framework}
For the sake of notational simplicity, throughout the remainder of this paper, we will denote the distance \eqref{eq:wasserstein} with $p=2$ by
\[
d(X,\sigma X) := W_2(X,\sigma X).
\]
Deterministic approaches, such as \cite{Cisternas2025}, must decide whether $d(X,\sigma X)$ is ``small enough'' by comparing it to an arbitrary threshold. Such methods are sensitive to noise and provide no measure of uncertainty. To overcome these limitations, we recast symmetry detection as a statistical inference problem. 
\subsection{The Cost Function and its Properties}
A central contribution of this work is the formulation of a cost function that reflects the degree of invariance of the observed data $X$ under a candidate symmetry group $\Sigma$. Intuitively, the cost should be low when $X$ is consistent with the hypothesis that the attractor is invariant under all transformations in $\Sigma$, and high otherwise.

To quantify this, we define a cost function $\mathcal{C}(\Sigma,X)$ as the average squared Wasserstein distance over the entire group orbit, i.e.
\begin{equation}\label{eq:cost}
    \mathcal{C}(\Sigma,X) = \frac{1}{|\Sigma|} \sum_{\sigma \in \Sigma} d(X,\sigma X)^2.
\end{equation}
This construction enforces the requirement that invariance must hold for every element of $\Sigma$. Notice that the cost function \eqref{eq:cost} has a built-in complexity penalty. Indeed, if $\Sigma'$ is a strict supergroup of the true symmetry $\Sigma$, then $\Sigma'$ introduces additional transformations for which $d(X,\sigma X)$ is large, thereby increasing the average cost. This mechanism implements a Bayesian form of \emph{Occam's razor} since, among competing hypotheses, the simplest group consistent with the data is favored (see, for example, Chapter~$4$ in \cite{Mackay2003}). This intuitive property can be formalized into a precise criterion that guarantees the preference for a simpler model, provided the data is sufficiently symmetric.
\begin{proposition}[Occam's Razor]\label{prop:occams_razor}
    Let $\Sigma_S, \Sigma_L$ be two candidate symmetry groups for $X$ with $\Sigma_S < \Sigma_L$. Then, one has
\begin{align}\label{eq:occams_razor}
        \mathcal C(\Sigma_S,X) < \mathcal C(\Sigma_L, X),
    \end{align}
if and only if the average cost of the transformations of $X$ contained in $\Sigma_S$ is less than the average cost of the transformations in $\Sigma_L \setminus \Sigma_S$.
\end{proposition}
\begin{proof}
    Let $C_S := \Sigma_{\sigma \in \Sigma_S} d(X,\sigma X)^2$ be the sum of squared distances for the group $\Sigma_S$ and put $C_{L\setminus S} := \Sigma_{\sigma \in \Sigma_L \setminus \Sigma_S} d(X,\sigma X)^2$ . With these notations, the inequality \eqref{eq:occams_razor} becomes
    \[
    \frac{C_S}{|\Sigma_S|} < \frac{C_S + C_{L\setminus S}}{|\Sigma_L|}.
    \]
Rearranging terms, one finds the equivalent condition
\[
(|\Sigma_L| - |\Sigma_S|) C_S < |\Sigma_S| C_{L \setminus S}.
\]
Finally, dividing both sides by the positive constants $|\Sigma_L| - |\Sigma_S|$ and $|\Sigma_S|$ (recall, $\Sigma_S$ is assumed to be a proper subgroup of $\Sigma_L$), one obtains the desired tolerance criterion
\[
\frac{C_S}{ |\Sigma_S|} < \frac{C_{L \setminus S}}{|\Sigma_L| - |\Sigma_S|}.
\]
\end{proof}
    \vs
Our cost function enjoys an additional structural property that guarantees consistency across physically equivalent datasets. Specifically, applying a transformation $g \in G$ to the data does not alter the evaluation of symmetry, provided the candidate subgroup is conjugated accordingly. Formally, for any $g \in G$ and any subgroup $\Sigma \le G$, one has
\[
\mathcal{C}(\Sigma, gX) = \mathcal{C}(g \Sigma g^{-1}, X).
\]
This \emph{conjugation equivariance} ensures that the cost depends only on the relative configuration of the data and the subgroup, not on any arbitrary choice of reference frame.  
\begin{proposition}[Conjugation Equivariance of Cost]\label{prop:conjugation_invariance}
    Let $G$ be any group acting orthogonally (isometrically) on $\br^d$. For any candidate subgroup $\Sigma \leq G$, transformation $g \in G$ and data set $X \subset \br^d$, one has
\[
\mathcal C(\Sigma,gX) = \mathcal C(g^{-1} \Sigma g,X).
\]
\end{proposition}
\begin{proof}
Let $X,Y \subset \br^d$, $(x,y) \in X \times Y$ and $g \in G$.
Since $\br^d$ is an orthogonal $G$-representation, one always has
    \[
    \| gx - gy \| = \| x - y \| \iff d(gX,gY) = d(X,Y).
    \]
In particular, for any $\sigma \in \Sigma$, it follows that
\begin{align}\label{eq:conjugation_invariance}
    d(gX, \sigma gX) = d(g^{-1}g X,g^{-1}\sigma g X) = d(X, g^{-1}\sigma g X).
\end{align}
Now, consider the cost function for $\Sigma$ evaluated on the transformed data $gX$
\[
\mathcal C(\Sigma,gX) = \frac{1}{|\Sigma|} \sum_{\sigma \in \Sigma} d(gX,\sigma gX)^2.
\]
Substituting the identity \eqref{eq:conjugation_invariance}, one obtains
\[
\mathcal C(\Sigma, gX) = \frac{1}{|\Sigma|} \sum_{\sigma \in \Sigma} d(X,g^{-1}\sigma gX)^2 = \mathcal C(g^{-1}\Sigma g, X).
\]
\end{proof}
\vs
We can also establish an analytical property of the cost function that guarantees its stability under perturbations of the data. This result is essential for practical applications, providing a rigorous basis for the method's robustness to measurement noise and enabling principled uncertainty quantification. Specifically, if two datasets $X$ and $Y$ are close in Wasserstein distance, then their associated costs $\mathcal{C}(\Sigma,X)$ and $\mathcal{C}(\Sigma,Y)$ differ by at most a controlled amount.
\begin{proposition}[Stability to Perturbations]\label{prop:stability}
Let $X,Y$ be datasets with $d:=d(X,Y)$. If there exists $M\ge 0$ satisfying
$d(X,\sigma X)\le M$ for all $\sigma\in\Sigma$, then
\[
\big|\mathcal{C}(\Sigma,X)-\mathcal{C}(\Sigma,Y)\big|\ \le\ 4Md + 4d^2.
\]
\end{proposition}
\begin{proof}
Fix $\sigma\in\Sigma$ and set $a:=d(X,\sigma X)$, $b:=d(Y,\sigma Y)$. Notice, by the triangle inequality, that one has
\[
a = d(X,\sigma X) \leq d(X,Y) + d(Y,\sigma X) = d + d(Y,\sigma X),
\]
and
\[
d(Y,\sigma X) \leq d(Y,\sigma Y) + d(\sigma Y, \sigma X) = b +  d(\sigma Y, \sigma X).
\]
Since $d(\sigma Y, \sigma X) = d(Y,X) = d$, this upper bound becomes $a \leq b + 2d$. 
Notice that, by symmetry, one also has $b\le a+2d$ and hence $|a-b|\le 2d$. Using $|a^2-b^2|=|a-b|\,(a+b)$ together with $a\le M$ and $b\le M+2d$ yields
\[
|a^2-b^2|\ \le\ (2d)\,(2M+2d) = 4Md + 4d^2.
\]
From here, averaging over $\sigma\in\Sigma$ provides the desired bound.
\end{proof}
\subsection{Posterior Distribution}
In order to move from the deterministic cost function \eqref{eq:cost} to a probabilistic framework we must specify a prior distribution $P(\Sigma)$ and characterize a posterior distribution $ P(\Sigma \mid X)$ over our set of candidate subgroups. Throughout this work, we assume a uniform prior, reflecting an unbiased initial belief by assigning equal a priori probability to all candidate groups. With this choice, the \textbf{Gibbs posterior} distribution is given by
\begin{equation}\label{eq:gibbs_posterior}
    P(\Sigma \mid X) \propto \exp\big(-\lambda \,\mathcal{C}(\Sigma,X)\big),
\end{equation}
where $\lambda > 0$ is a hyperparameter that acts as an inverse temperature. In our setting, larger values of $\lambda$ impose a stronger penalty on non-zero costs, causing the posterior to concentrate sharply on the group that minimizes the cost. This approach allows us to assign a probability to each symmetry hypothesis $\Sigma$, such that lower-cost hypotheses are assigned higher probabilities. 
\subsection{The MCMC Sampler}
For any infinite lattice of candidate subgroups, the posterior distribution \eqref{eq:gibbs_posterior} cannot be directly computed, so we use a Metropolis--Hastings Markov Chain Monte Carlo (MCMC) algorithm to approximate it. Assume in particular that the parameter space is the discrete lattice of candidate subgroups, e.g. the integers $n \geq 2$ for dihedral groups $D_n$ (see Appendix \ref{app:dihedral}).

Starting from an initial group $\Sigma$, the sampler proposes a new candidate $\Sigma'$ by a local move in this lattice (e.g., $n \mapsto n \pm 1$). The proposed move is accepted with probability
\begin{equation}\label{eq:acceptance}
    \alpha(\Sigma' \mid \Sigma) = \min\!\Bigg( 1,\; \frac{P(X \mid \Sigma')}{P(X \mid \Sigma)} \Bigg).
\end{equation}
The resulting Markov chain is a sequence of subgroup samples whose empirical distribution approximates the posterior $P(\Sigma \mid X)$. From these samples, we estimate posterior probabilities, compute maximum a posteriori (MAP) estimates, and quantify uncertainty over the symmetry structure.

\section{Numerical Validation}\label{sec:validation}

To assess the robustness and practical utility of our Bayesian framework, we validate the method on several examples, including a trajectory from an equivariant dynamical systems, a synthetic point cloud and a real-world data set.

\subsection{Example 1: Dihedral Symmetries in The Chossat-Golubitsky System}

In their seminal work \cite{ChossatGolubitsky88}, Chossat and Golubitsky analyzed \emph{symmetry-increasing bifurcations} for a family of $D_n$-equivariant maps on the complex plane ($z \in \mathbb{C}$) of the form
\begin{equation}\label{eq:cg_map}
    f(z) = \big(\alpha |z|^2 + \beta (z^n + \bar{z}^n)/2 + \lambda\big)z + \gamma \bar{z}^{\,n-1}, 
\end{equation}
where `$|z|$' indicates the complex modulus of $z \in \bc$ and `$\overline{z}$' indicates its complex conjugation. For parameters $\alpha = 1$, $\beta = 0$, $\gamma = 0.5$, $\lambda = -1.804$, and $n = 3$, all trajectories of the discrete system $z_{k+1} = f(z_k)$ converge to an attractor with $D_3$ symmetry. This noise-free attractor serves as the ground truth for our first validation experiment (Figure~\ref{fig:cg_data}, left).

To simulate a challenging scenario, we generated a short trajectory of $N = 150$ points and added substantial isotropic Gaussian noise ($\sigma = 0.5$). As shown in Figure~\ref{fig:cg_data} (right), this noise level completely obscures the underlying $D_3$ symmetry to the naked eye, creating a difficult test for our method.
\begin{figure}[htbp]
    \centering
\includegraphics[width=\textwidth]{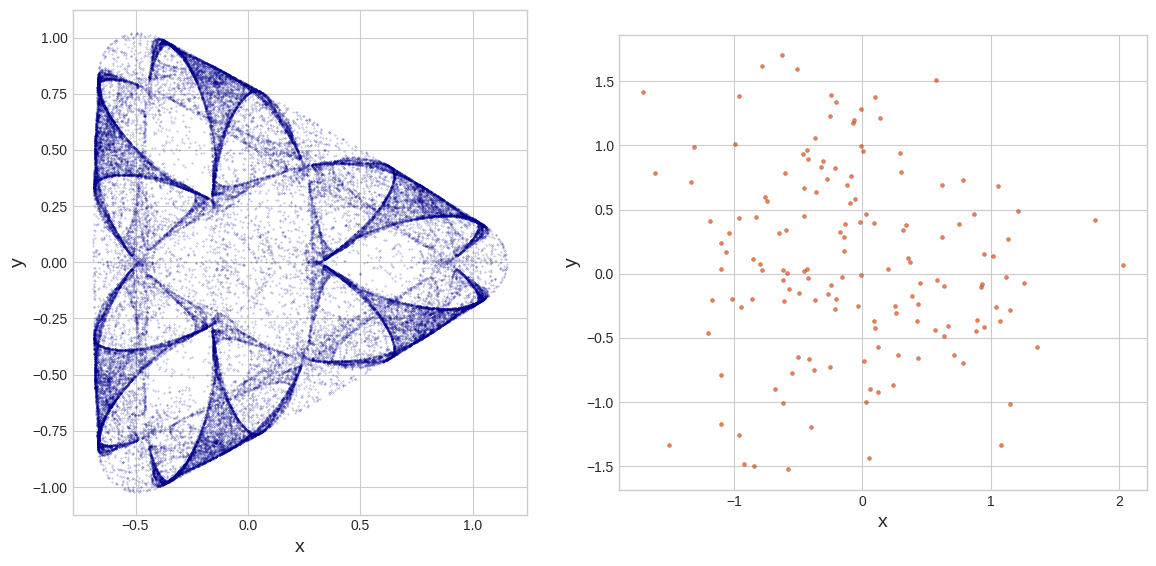}
\caption{\textbf{(Left)} Noise-free attractor for the Chossat--Golubitsky map, generated from a $50,000$-point trajectory to illustrate the ground-truth $D_3$ symmetry. \textbf{(Right)} Noisy ($\sigma = 0.5$) $150$ point trajectory used for analysis, where the symmetry is visually obscured.}
    \label{fig:cg_data}
\end{figure}
%\FloatBarrier
\subsubsection{A Deterministic Benchmark}
As a baseline, we applied the deterministic threshold-based method of Cisternas~\cite{Cisternas2025} to this noisy dataset. The procedure computes the Wasserstein distance for every transformation in a candidate group (e.g., $D_n$) and accepts the group only if \emph{all} element-wise distances fall below a user-defined threshold $\upsilon$.

We tested candidate groups $D_n$ with $2 \leq n \leq 6$. Results in Figure~\ref{fig:cg_benchmark} highlight the fragility of this approach: correct classification occurs only if $\upsilon$ lies within a narrow ``robust window'' between approximately $0.349$ and $0.391$. Below this range, the true symmetry ($D_3$) is rejected (\emph{false negative}); above it, non-symmetric transformations are accepted (\emph{false positive}). 
This extreme sensitivity to a single, user-defined hyperparameter highlights the need for a more robust method.

\begin{figure}[htbp]
    \centering    \includegraphics[width=\textwidth]{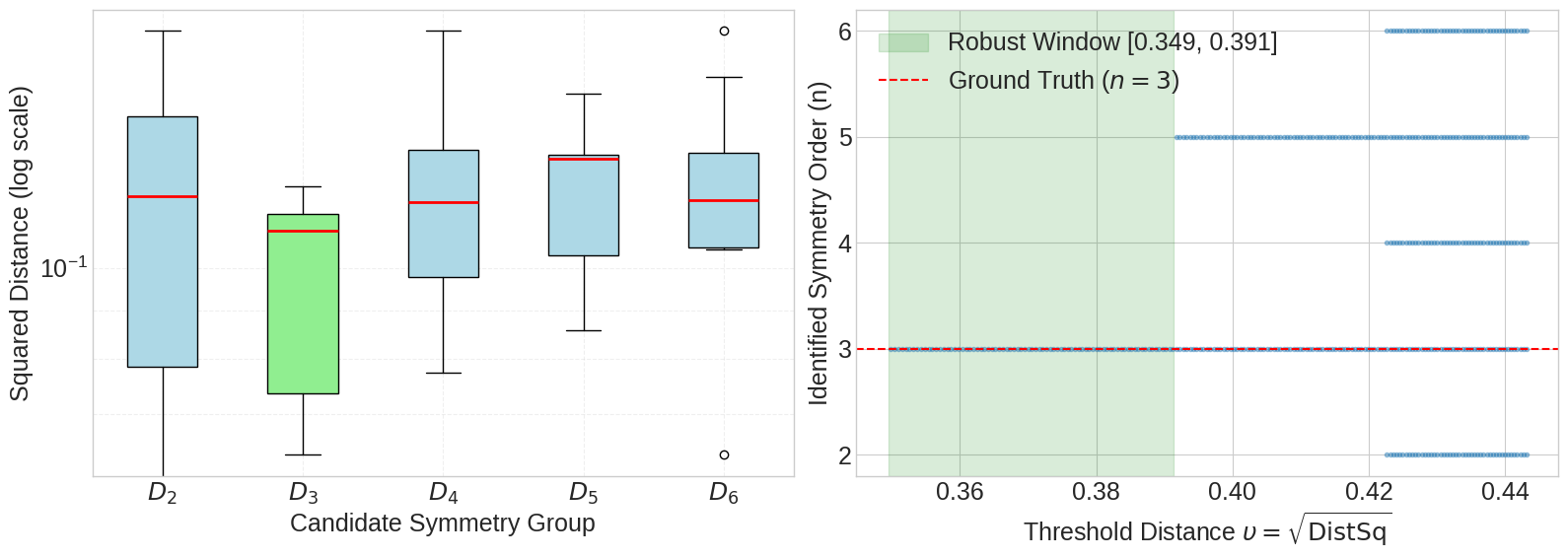}
    \caption{Benchmark analysis of the deterministic method from \cite{Cisternas2025} on the noisy $D_3$ attractor. \textbf{(Left)} Distribution of squared Wasserstein distances for candidate groups; distances for $D_3$ cluster near zero. \textbf{(Right)} Classification outcome as a function of threshold $\upsilon$. Correct identification of $n = 3$ occurs only within the narrow green ``robust window''.}
    \label{fig:cg_benchmark}
\end{figure}
\subsubsection{Bayesian Inference Results}
Applying our Bayesian framework to the same dataset, we ran a Metropolis--Hastings sampler for $20,000$ iterations with sharpness parameter $\lambda = 250$. The resulting posterior distribution $P(n \mid X)$ is shown in Figure~\ref{fig:cg_mcmc}.

The \emph{maximum a posteriori} (MAP) estimate is $n = 3$, correctly recovering the ground-truth symmetry despite severe noise and limited data. The posterior assigns $72.8\%$ probability to this model, reflecting the high uncertainty inherent in the task. Rather than overconfidently committing to a single label, the posterior spreads remaining mass across higher-order symmetries, providing an honest uncertainty quantification. The MCMC trace (Figure~\ref{fig:cg_mcmc}, left) confirms effective exploration (achieving an acceptance rate of $0.205$) of the parameter space before concentrating near $n = 3$.

\begin{figure}[htbp]
    \centering
    \includegraphics[width=\textwidth]{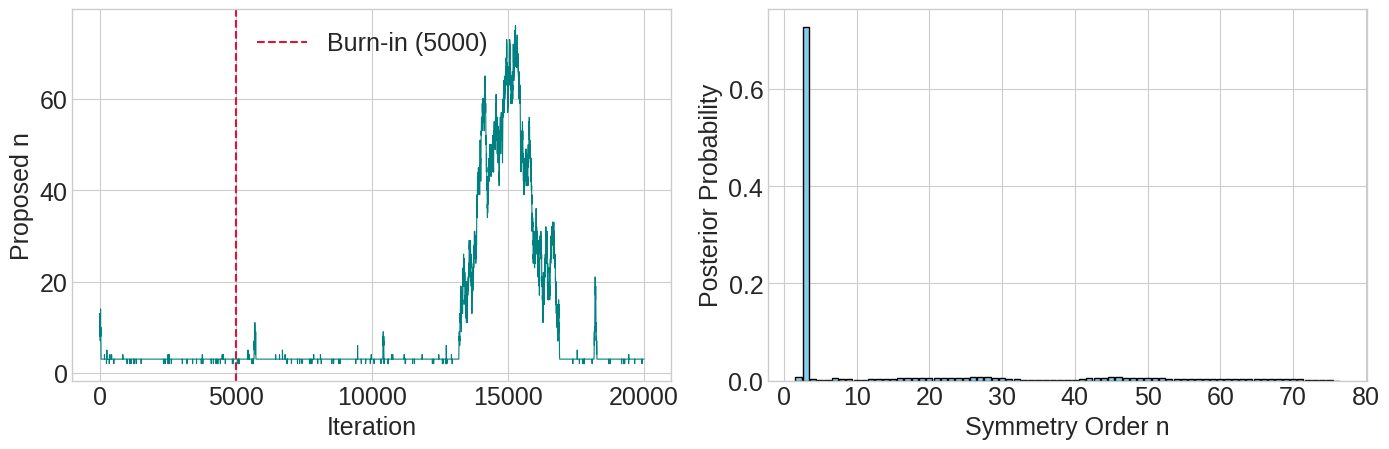}
    \caption{Bayesian inference on the noisy $D_3$ attractor. \textbf{(Left)} MCMC trace for symmetry order $n$. \textbf{(Right)} Posterior distribution $P(n \mid X)$, with MAP at $n = 3$ (posterior probability $72.8\%$).}
    \label{fig:cg_mcmc}
\end{figure}
\FloatBarrier
\subsection{Example 2: Nested Dihedral Symmetries in Synthetic Point Cloud}
To evaluate the framework under more challenging conditions, we constructed a synthetic dataset with high-order dihedral symmetry $D_{12}$, which admits a rich lattice of subgroups ($D_6, D_4, D_3, D_2$). Rather than generating a trajectory from a known map, we built the point cloud directly to ensure perfect underlying symmetry. Specifically, we sampled a ``motif'' of eight random points within a fundamental domain of the $D_{12}$ action and replicated this motif under all $2 \times 12 = 24$ group transformations, producing a $D_{12}$-invariant point cloud of $N = 192$ points. To simulate realistic conditions, we added isotropic Gaussian noise with $\sigma = 0.05$ (Figure~\ref{fig:d12_data}).
\begin{figure}[htbp]
    \centering
\includegraphics[width=\textwidth]{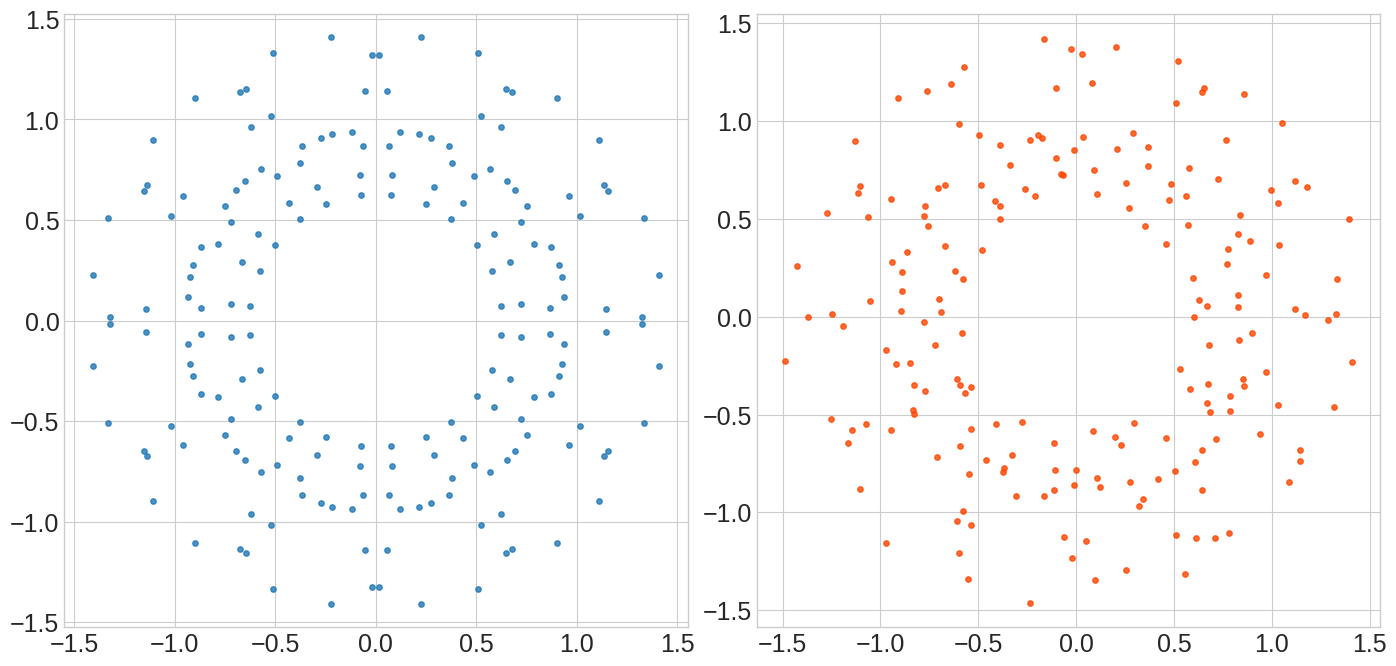}
    \caption{\textbf{(Left)} Noise-free point cloud with perfect $D_{12}$ symmetry. \textbf{(Right)} Noisy dataset ($N = 192$, $\sigma = 0.05$) used for analysis.}
    \label{fig:d12_data}
\end{figure}

This scenario is particularly challenging because any dataset with $D_{12}$ symmetry is also perfectly invariant under the transformations of all its subgroups. This creates a highly multimodal posterior landscape, where $n=2, 3, 4, 6,$ and $12$ are all valid hypotheses.

\subsubsection{A Deterministic Benchmark}
Applying the deterministic threshold-based method from \cite{Cisternas2025} to this dataset further highlights its limitations. Figure~\ref{fig:d12_benchmark} (left) shows that all subgroups of $D_{12}$ exhibit similarly low costs, making them indistinguishable from the ground truth. The threshold sweep (right) confirms the absence of any ``robust window'' isolating $n = 12$: thresholds low enough to accept $D_{12}$ also accept its subgroups, while thresholds high enough to isolate a single group select one of the simpler subgroups instead.

\begin{figure}[htbp]
    \centering
\includegraphics[width=\textwidth]{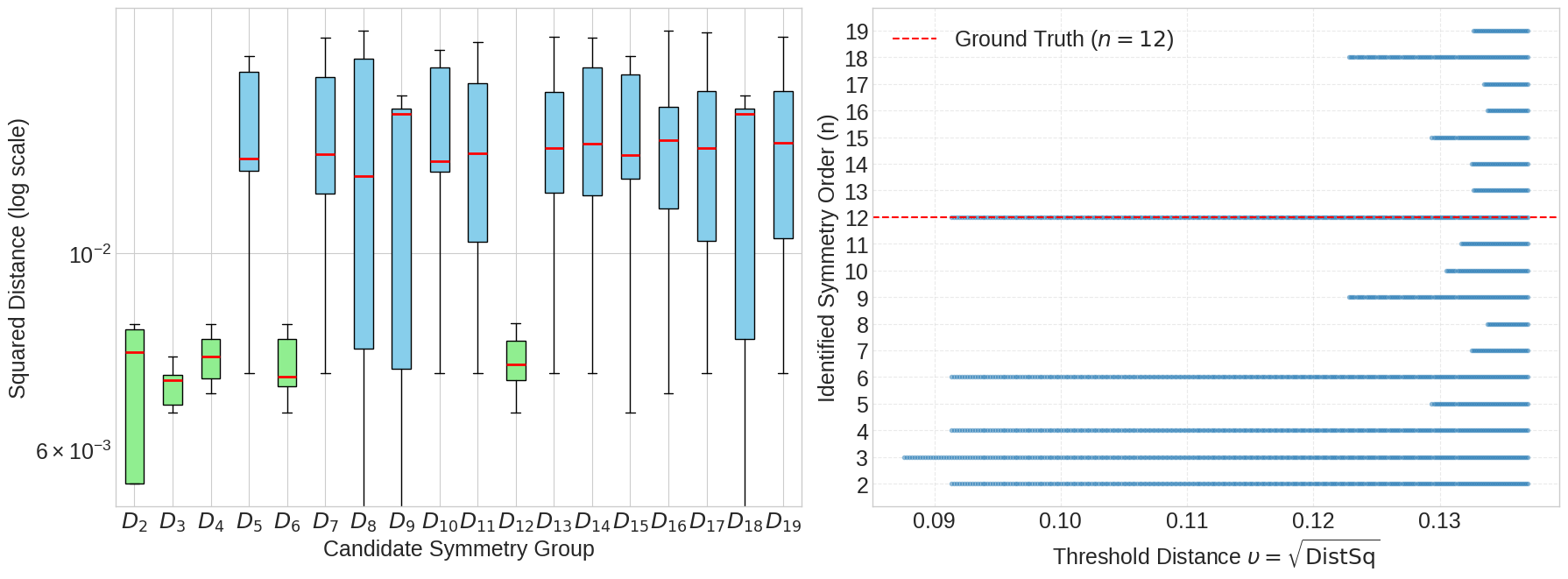}
    \caption{Benchmark analysis from \cite{Cisternas2025} on the noisy $D_{12}$ dataset. \textbf{(Left)} Boxplot of squared Wasserstein distances for candidate groups; all subgroups of $D_{12}$ (highlighted in green) exhibit low costs. \textbf{(Right)} Threshold sweep fails to isolate the true symmetry $n = 12$.}
    \label{fig:d12_benchmark}
\end{figure}

\subsubsection{Bayesian Inference Results}
The multimodal posterior landscape poses a challenge for standard MCMC samplers, which can become trapped in local modes corresponding to subgroups. To address this, we employed a \emph{compound proposal mechanism}: at each step, the sampler performs a local move ($n \mapsto n \pm 1$) with probability $95\%$ or a global ``jump'' to a randomly chosen $n$ with probability $5\%$. This strategy enables the chain to escape local maxima and explore the full posterior.

We ran the sampler for $100,000$ iterations with $\lambda = 250$. Results are shown in Figure~\ref{fig:d12_mcmc}. The trace plot indicates excellent mixing (acceptance rate $0.728$), and the posterior distribution reflects the inherent ambiguity of the problem: rather than collapsing to a single peak, it assigns substantial probability mass to the true symmetry $D_{12}$ and its valid subgroups ($D_6, D_4, D_3, D_2$). This illustrates the framework's ability to capture structural uncertainty in cases where deterministic methods fail.

\begin{figure}[htbp]
    \centering
    \includegraphics[width=\textwidth]{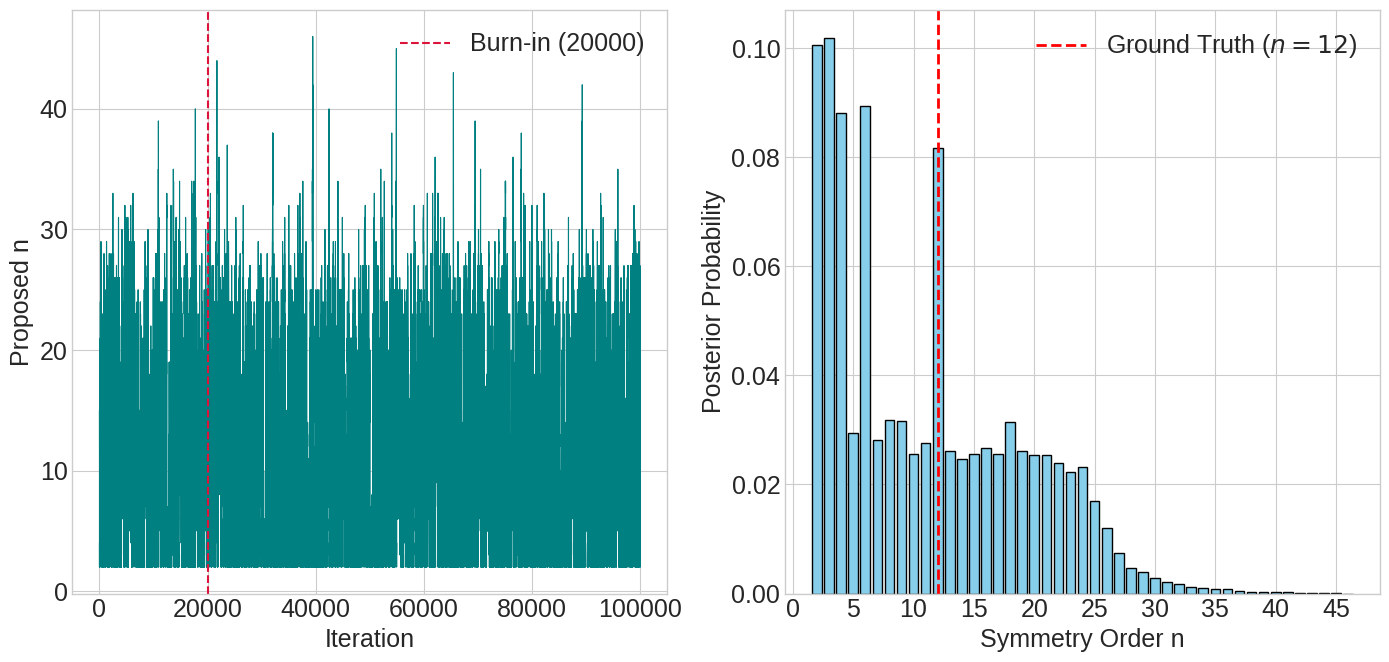}
    \caption{Bayesian inference results for the $D_{12}$ dataset. \textbf{(Left)} The MCMC chain trace demonstrates excellent mixing due to the jump proposal mechanism. \textbf{(Right)} The posterior distribution correctly identifies the family of candidate symmetries, assigning probability mass to the true group ($n=12$) and its subgroups.}
    \label{fig:d12_mcmc}
\end{figure}
\FloatBarrier
\subsection{An Application to Real-World 
Biomechanical Time Series Data}
To demonstrate the framework's utility on real-world data, we analyze a public dataset of human gait dynamics from the UCI Machine Learning Repository \cite{helwig2016smoothing}. The dataset consists of joint-angle time series from ten healthy subjects walking under three conditions: normal (unbraced), with a restrictive brace on the right knee, and with a restrictive brace on the right ankle. For each condition, ten consecutive gait cycles were recorded for both legs across the ankle, knee, and hip joints, with each cycle normalized to $101$ time points. The central question is whether our equation-free framework can detect and quantify changes in the underlying dynamic symmetries of the locomotor system induced by these mechanical constraints.

\vspace{0.5em}
Our method operates on geometric invariances of $2$D point clouds, whereas the gait data consists of $1$D time series. To bridge this gap, we embed each time series into the complex plane using the \textbf{Hilbert transform}. For a signal $A(t)$, the Hilbert transform produces the analytic signal
\[
Z(t) := A(t) + i\,H(A(t)),
\]
where $H(A(t))$ is the quadrature component phase-shifted by $90$ degrees. From $Z(t)$, we extract the \textbf{instantaneous phase} $\varphi(t) := \arg(Z(t))$. Mapping $t \mapsto e^{i\varphi(t)}$ projects the temporal pattern onto the unit circle, converting the $1$D signal into a $2$D geometric trajectory whose symmetries reflect underlying periodic structures of the original time series.
\subsubsection{Bayesian Inference Results}

Following exploratory analysis, we selected a subset of trials that best distinguished the three conditions: the averaged ankle–knee dynamics of the right leg for subjects $4,7,8,9,10$ (Figure~\ref{fig:gait_embedding}). For each condition, we averaged these trials to obtain a low-noise representative gait cycle.
\begin{figure}[htbp]
    \centering
\includegraphics[width=\textwidth]{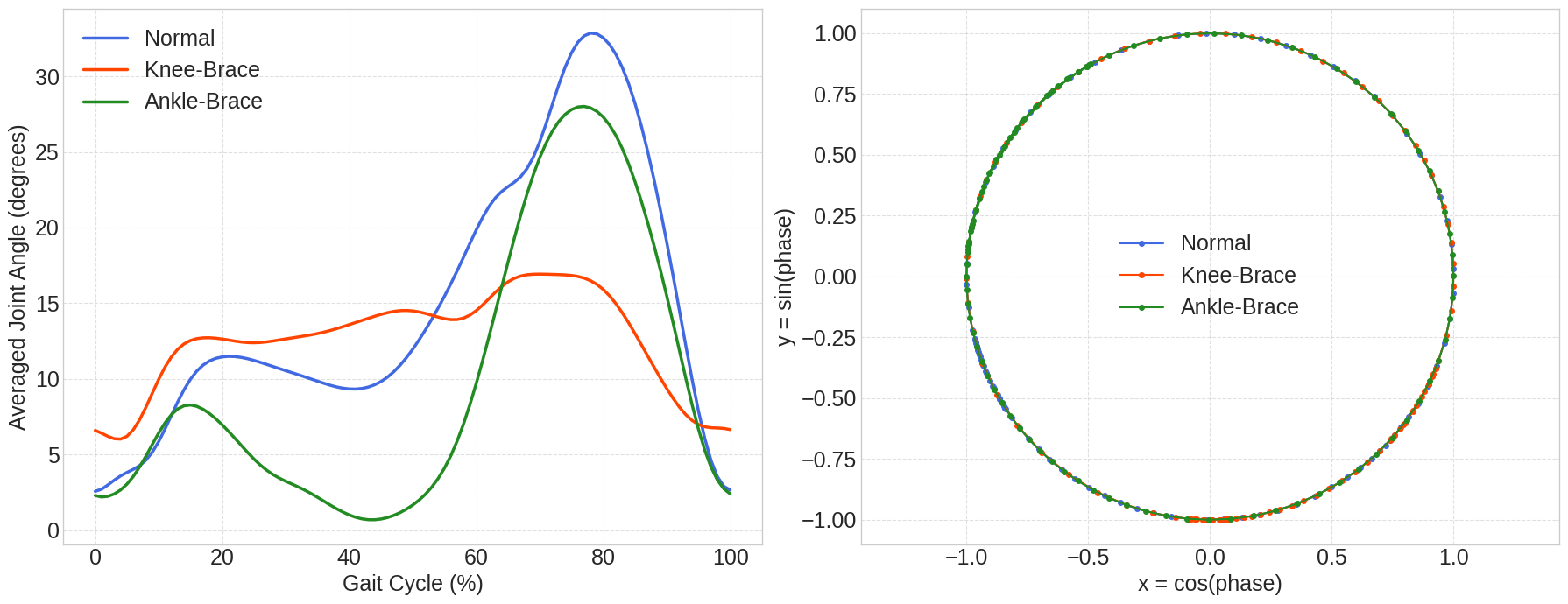}
    \caption{\textbf{(Left)} Representative time series for each condition, obtained by averaging ankle–knee dynamics over selected trials. \textbf{(Right)} Corresponding $2$D embeddings via Hilbert-phase mapping, used for symmetry analysis.}
\label{fig:gait_embedding}
\end{figure}
\vs
The resulting posterior landscapes were sharply peaked, making standard MCMC inefficient. We therefore employed \textbf{Metropolis-Coupled MCMC (MC\textsuperscript{3})}, or Parallel Tempering, which runs multiple chains at different ``temperatures'' (flattened posteriors) and allows swaps between chains. This enables the cold chain to escape local maxima and explore the full parameter space. We further incorporated a hybrid proposal mechanism combining local moves ($n \mapsto n \pm 1$) with occasional global jumps.

Each condition was analyzed with a five-chain geometric temperature ladder for $50,000$ iterations, using the uniform sharpness parameter $\lambda = 200$. Results reveal a clear, hierarchical shift in dynamic symmetry under mechanical constraints.

\subsubsection{Normal Walking}
For the unbraced condition (Figure~\ref{fig:gait_normal}), the posterior is sharply peaked at $n = 2$, with MAP estimate $D_2$ and posterior probability $75.5\%$. This indicates a strong bilateral symmetry, consistent with the fundamental alternation of limb motion during gait. The sampler exhibited excellent mixing (local acceptance rate $0.242$).

\begin{figure}[htbp]
    \centering
\includegraphics[width=\textwidth]{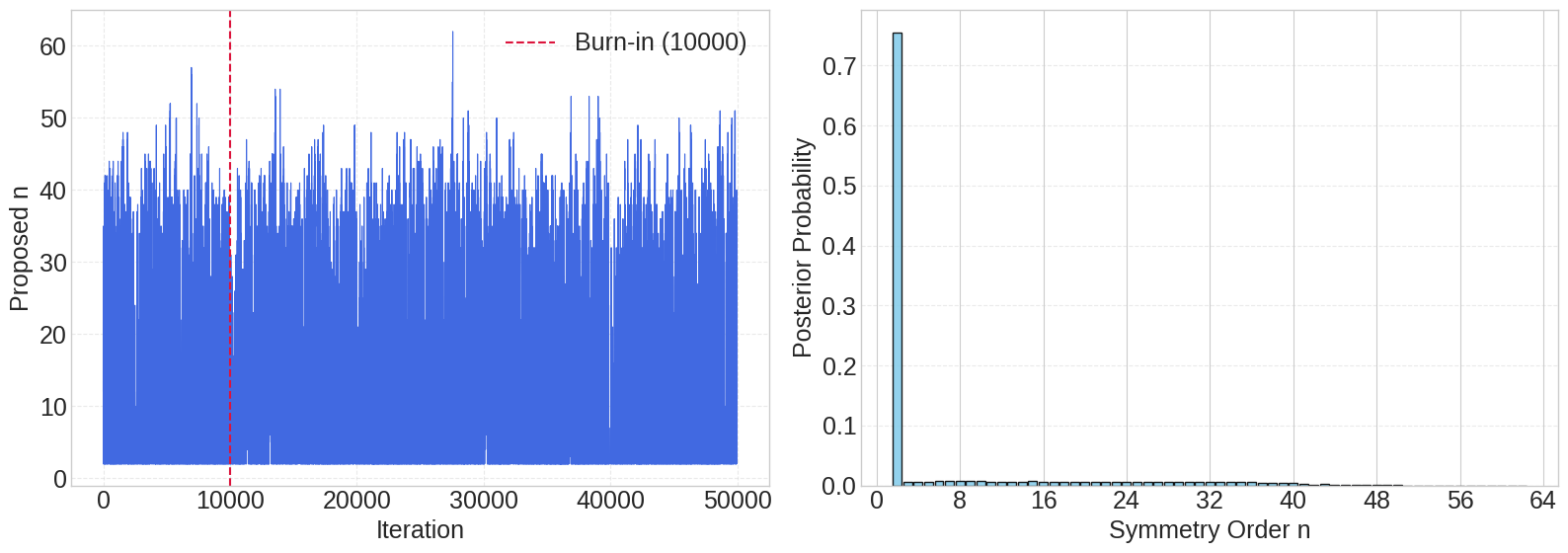}
    \caption{Posterior for the \textbf{Normal} condition. A dominant peak at $n = 2$ indicates bilateral ($D_2$) symmetry.}
    \label{fig:gait_normal}
\end{figure}

\subsubsection{Knee-Brace Condition}
When the knee is braced (Figure~\ref{fig:gait_knee}), the underlying symmetry fundamentally changes. The posterior distribution shows a clear peak at $\mathbf{n=3}$ with a MAP confidence of \textbf{$10.9\%$}. This demonstrates that the mechanical constraint of the knee brace forces the coordinative dynamics of the ankle-knee subsystem to reorganize from a simple bilateral pattern into a more complex, tri-modal rhythmic structure. The sampler remains well-behaved, with a local acceptance rate of $0.900$.

\begin{figure}[htbp]
    \centering
\includegraphics[width=\textwidth]{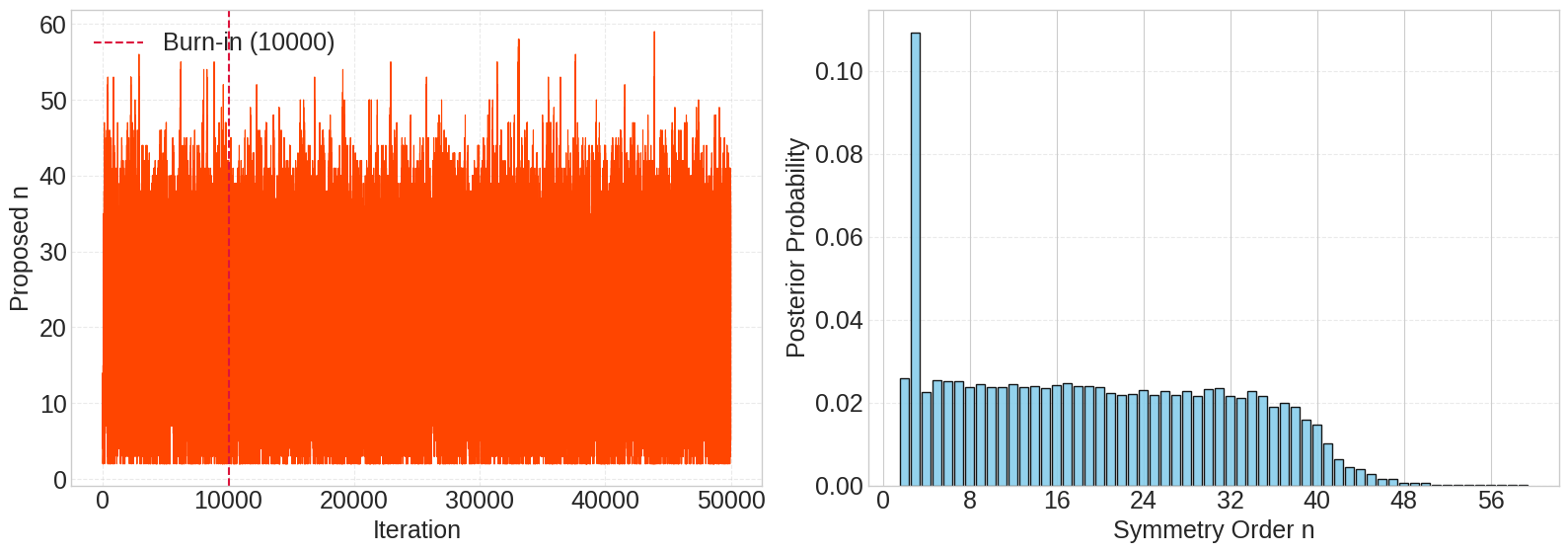}
    \caption{Posterior for the \textbf{Knee-Brace} condition. The MAP estimate shifts to $n = 3$, indicating a $D_3$ symmetry pattern.}
    \label{fig:gait_knee}
\end{figure}

\subsubsection{Ankle-Brace Condition}
Constraining the ankle, a joint that is critical for propulsion, induces an even greater shift (Figure~\ref{fig:gait_ankle}). The posterior is dominated by $n = 4$ (MAP probability $3.8\%$), with a secondary peak at $n = 3$, reflecting residual tri-modal tendencies (local acceptance $ 0.972$). This demonstrates the framework's ability to capture both dominant and competing symmetry structures.

\begin{figure}[htbp]
    \centering
\includegraphics[width=\textwidth]{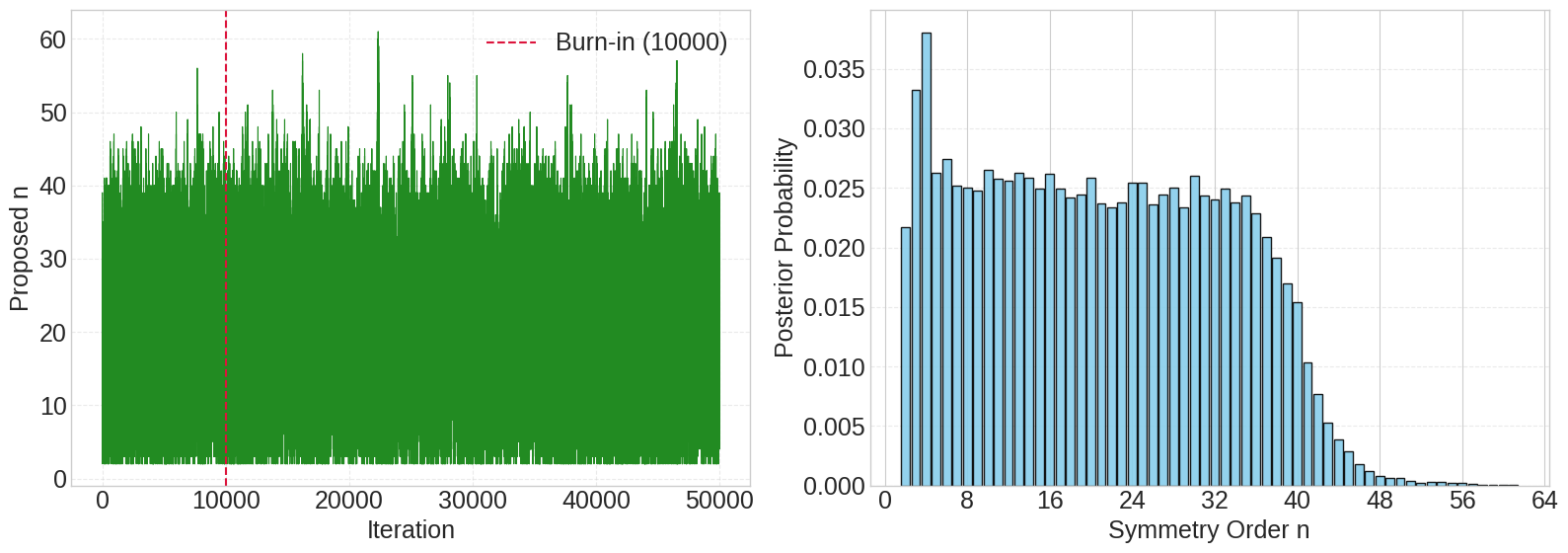}
    \caption{Posterior for the \textbf{Ankle-Brace} condition. The MAP estimate is $n = 4$ ($D_4$ symmetry), with a secondary peak at $n = 3$ indicating uncertainty.}
    \label{fig:gait_ankle}
\end{figure}

\appendix
\section{A Primer on Transformation Groups}\label{sec:transformation_groups}
In this section, we collect some notions from the theory of transformation groups (for a more thorough treatment, we refer the reader to \cite{Kawakubo}). Let $G$ be a group. A set $H \subseteq G$ is said to be a \textbf{subgroup} of $G$, denoted $H \leq G$, if for every $a,b \in H$ one has $a * b^{-1} \in H$, where `$*$' indicates the group operation.
A set $X$ is said to be a \textbf{$G$-set} if there exists
a \textbf{$G$-action} on $X$, i.e. a continuous map $ \rho :G \times X \to X$, satisfying the properties
\begin{enumerate}
    \item $\rho(e,x) \to x$ for all $x \in X$, where $e$ is the identity element in $G$;
    \item $\rho(g_2\rho(g_1,x)) \to \rho(g_2 g_1, x)$ for all $g_1,g_2 \in G$ and $x \in X$.
\end{enumerate}
For notational simplicity, we adopt the convention
\[
gx : = \rho (g,x) \text{ for all } g \in G \text{ and } x \in X. 
\]
A topological space $X$ equipped with a $G$-action is also called a \textbf{$G$-space}. 
Supposing that a $G$-space $X$ admits a metric $d:X \times X \to [0,\infty)$, the action of $G$ on $X$ is said to be \textbf{isometric} (or \textbf{orthogonal} in the case that $X$ is Euclidean) if
\[
d(gx,gy) = d(x,y) \text{ for all } x,y \in X \text{ and } g \in G.
\]
For each $x \in X$, the set $G(x) := \{ gx : g \in G \}$ is called the \textbf{orbit} of $x$ in $X$ and the subgroup $G_x := \{ g \in G : gx = x \}$ is called the \textbf{isotropy} of $x$ in $G$. A subset $D \subset X$ is said to be a \textbf{fundamental domain} for the $G$-action if the following two conditions are satisfied: 
\begin{enumerate}
    \item \textbf{Coverage:} for every $x \in X$, there exist $g \in G$ and $d \in D$ such that $x = gd$;
    \item \textbf{Minimality:} for each $d \in D$, one has $gd = d$ if and only if $g = e$. 
\end{enumerate}
\begin{remark}
Notice that any subset of the general linear group $GL(n,\br)$ acts on the Euclidean space $\br^d$ and that any subgroup of the orthogonal group $O(n)$ acts on $\br^d$ isometrically.  
\end{remark}

\subsection{The Dihedral Group and Its Action on the Complex Plane}\label{app:dihedral}
The dihedral group $D_n$ is the symmetry group of a regular $n$-gon. It has order $2n$ admits the presentation
\[
D_n = \langle r, s \;\mid\; r^n = e,\; s^2 = e,\; srs = r^{-1} \rangle,
\]
where $r$ represents a rotation by angle $2\pi/n$ and $s$ represents a reflection across an axis of symmetry. Every element of $D_n$ can be written uniquely in the form $r^k$ or $r^k s$, with $k \in \{0,1,\dots,n-1\}$.

\subsubsection{Matrix and complex models}
Identifying $\bc$ with $\br^2$, one can view $D_n$ as a subgroup of the orthogonal group $O(2)$. In matrix form, the 
rotation generator $r \in D_n$ acts on vectors in $\br^2$ via
\[
R := \begin{pmatrix}
\cos(2\pi/n) & -\sin(2\pi/n)\\[2pt]
\sin(2\pi/n) & \phantom{-}\cos(2\pi/n)
\end{pmatrix},
\qquad
R^k = \begin{pmatrix}
\cos(2\pi k/n) & -\sin(2\pi k/n)\\[2pt]
\sin(2\pi k/n) & \phantom{-}\cos(2\pi k/n)
\end{pmatrix},
\]
while the reflection $s \in D_n$ can be taken, for instance, as reflection across the $x$-axis,
\[
S := \begin{pmatrix}
1 & 0\\
0 & -1
\end{pmatrix}.
\]
These matrices satisfy $SRS = R^{-1}$, as required by the defining relations of $D_n$. 
\vs
For computations on the complex plane, it is often more convenient to use the complex representation. Let $\zeta_n := e^{2\pi i/n}$ denote a primitive $n$-th root of unity. The action of $D_n$ on $z \in \bc$ is given by
\begin{align*}
    \begin{cases}
        r_k z = \zeta_n^{\,k} z, \quad &k = 0,1,\dots,n-1; \\
        s_k z = \zeta_n^{\,k} \,\overline{z}, \quad &k = 0,1,\dots,n-1.
    \end{cases}
\end{align*}
The assignments $r \mapsto (z \mapsto \zeta_n z)$ and $s \mapsto (z \mapsto \overline{z})$ extends to a faithful homomorphism $D_n \to \mathrm{Isom}(\bc)$.

\subsubsection{Orbits, Isotropies, and Fundamental Domains}
For $z \in \bc$, denote the orbit and isotropy by
\[
G(z) := \{ g z : g \in D_n \}, 
\qquad 
G_z := \{ g \in D_n : g z = z \}.
\]
The structure of these sets depends on the position of $z$. Indeed, if $z \neq 0$ does not lie on any reflection axis, then $G(z)$ is trivial and $|G_z| = 2n$. If $z \neq 0$ lies on a reflection axis, then $G_z$ is generated by the corresponding reflection (i.e. $|G_z| = 2$) and $|G(z)| = n$. At the origin $z=0$, the isotropy is all of $D_n$ and the orbit is a singleton $\{0\} \subset \bc$.

A convenient fundamental domain for the action of $D_n$ on $\bc\setminus\{0\}$ is the closed sector
\[
\mathcal{F} := \bigl\{\, r e^{i\theta} : r \ge 0,\; \theta \in [\,0,\pi/n\,) \,\bigr\},
\]
with boundary rays identified by reflections. Every $z \in \bc\setminus\{0\}$ is $D_n$-conjugate to a unique point of $\mathcal{F}$ up to boundary identifications. 

\subsubsection{Orthogonality and metric invariance}
Since $D_n \leq O(2)$, its action is by Euclidean isometries. For all $g \in D_n$ and $x,y \in \bc$,
\[
|g x - g y| = |x - y|.
\]
Consequently, for empirical measures $\mu_X = \tfrac{1}{N}\sum_{i=1}^N \delta_{x_i}$ and $\mu_Y = \tfrac{1}{N}\sum_{j=1}^N \delta_{y_j}$, the $p$-Wasserstein distance is invariant under the action of $D_n$:
\[
W_p\bigl(gX,\, gY\bigr) = W_p(X,Y), \qquad g \in D_n.
\]
This invariance underpins the conjugation-invariance properties used in Section~\ref{sec:bayesian_framework} and justifies computing group-orbit cost using Euclidean metrics.

\subsubsection{$D_n$ Equivariance of maps}
A map $f:\bc\to\bc$ is $D_n$-equivariant if $f(g z) = g f(z)$ for all $g \in D_n$ and $z \in \bc$. In the complex model above, this requires
\[
f(\zeta_n^{\,k} z) = \zeta_n^{\,k} f(z), 
\qquad 
f(\zeta_n^{\,k} \overline{z}) = \zeta_n^{\,k}\, \overline{f(z)}.
\]
These constraints determine the admissible terms in normal forms for $D_n$-equivariant dynamics and connect directly to the symmetry tests employed in Section~\ref{sec:validation}.

% \begin{thebibliography}{99}
\bibliographystyle{plain}
\bibliography{reference}

‌
% \bibitem{Mackay2003} D.J.C. MacKay, \emph{Information Theory, Inference, and Learning Algorithms}. Cambridge University Press, 2003.

% \bibitem{helwig2016smoothing} N. Helwig \& E. Hsiao-Wecksler,  Multivariate Gait Data [Dataset], UCI Machine Learning Repository (2016). https://doi.org/10.24432/C5861T.

% \bibitem{Melbourne1993} I. Melbourne, M. Dellnitz, \&  M. Golubitsky, \emph{The structure of symmetric attractors}, Archive for Rational Mechanics and Analysis, 123(1) (1993), 75–98. https://doi.org/10.1007/bf00386369

% \bibitem{Olver} P.J. Olver, \emph{Applications of Lie Groups to Differential Equations}, Springer-Verlag, (1993).

% \bibitem{Peyre} G. Peyré, M. Cuturi, \emph{Computational Optimal Transport: With Applications to Data Science}, Foundations and Trends in Machine Learning, 11(5-6) (2019), 355–607. https://doi.org/10.1561/2200000073

% \end{thebibliography} 

\end{document}